\documentclass[letterpaper]{article}

\usepackage{aaai}
\usepackage{times}
\usepackage{helvet}
\usepackage{courier}
\usepackage{graphicx}
\usepackage{amsmath}
\usepackage{amssymb}
\usepackage{amsthm}
\usepackage{mathrsfs}
\usepackage{booktabs}

\newcommand{\field}[1]{\mathbb{#1}}

\newcommand{\set}[1]{\mathcal{#1}}

\newcommand{\reals}{\field{R}}

\newcommand{\comment}[1]{}

\newcommand{\prob}{\mathsf{P}}      
\newcommand{\expect}{\mathsf{E}}               
\newcommand{\simplex}[1]{\Delta(\set{X})}

\theoremstyle{plain}
\newtheorem{theorem}{Theorem}
\newtheorem{lemma}[theorem]{Lemma}

\newtheorem{corollary}[theorem]{Corollary}

\theoremstyle{definition}

\title{An Adversarial Interpretation of
Information-Theoretic Bounded Rationality}
\author{Pedro A. Ortega \and Daniel D. Lee
\\School of Engineering and Applied Sciences
\\University of Pennsylvania
\\Philadelphia, PA 19104, USA
\\ \{ope,ddlee\}@seas.upenn.edu}

\begin{document} 

\frenchspacing

\maketitle

\begin{abstract}
\begin{quote}
Recently, there has been a growing interest in modeling planning with
information constraints. Accordingly, an agent maximizes a regularized
expected utility known as the free energy, where the regularizer
is given by the information divergence from a prior to a posterior policy. While
this approach can be justified in various ways, including from
statistical mechanics and information theory, it is still unclear how it relates
to decision-making against adversarial environments. This connection
has previously been
suggested in work relating the free energy to risk-sensitive control
and to extensive form games. Here, we show that a single-agent free
energy optimization is equivalent to a game between the agent and an imaginary
adversary.  The adversary can, by paying an exponential
penalty, generate costs that diminish the decision maker's payoffs. It turns out
that the optimal strategy of the adversary consists in choosing costs so as to
render the decision maker indifferent among its choices, which is a definining
property of a Nash equilibrium, thus tightening the connection between free
energy optimization and game theory.
\end{quote}
\end{abstract}


\section{Introduction}\label{sec:introduction}

Recently, there has been a renewed interest in \emph{bounded
rationality}, originally proposed by Herbert A. Simon as an alternative
account to the model of perfect rationality in the face of
complex decisions \cite{Simon1972}. In artificial intelligence (AI), this
development has been largely motivated by the intractability of exact
planning in complex and uncertain environments
\cite{PapadimitriouTsitsiklis1987} and the difficulty of finding domain-specific
simplifications or approximations that would render planning tractable
\comment{\cite{Ortega2011}} despite the latest theoretical advancements in
understanding
the AI problem \cite{Hutter2004}. Newly proposed techniques based on Monte Carlo
simulations such as Monte Carlo Tree Search \cite{Coulom2006,Browne2012},
Thompson sampling \cite{Strens2000} and Free Energy
\cite{Kappen2005b} have become the focus of much AI and reinforcement
learning research due to their wide applicability and their unprecedented
success in difficult problems such as computer Go \cite{Gelly2006}, universal
planning \cite{Veness2011}, human-level video game playing \cite{Mnih2013} and
robot learning \cite{Theodorou2010}. Roughly, these efforts can be broadly
classified into two groups: (a) randomized approximations to the perfect
rationality model and (b) exact statistical-mechanical or information-theoretic
(IT) approaches. The latter group is of special theoretical interest, as it
rests on a model of bounded rationality\footnote{There are several
approaches to bounded rationality in the literature, most notably those coming
from the field of behavioral economics \cite{Rubinstein1998}, which
put emphasis on the procedural elements of decision making. Here we have settled
on the qualifier ``information-theoretic'' as a way to distinguish from these
approaches.} that yields randomized optimal policies as a consequence of
resource-boundedness \cite{OrtegaBraun2013}.

The difference between the perfectly rational and
the IT bounded rational model lies in the choice of the
objective function: the IT approach adds a regularization
term in the form of a Kullback-Leibler divergence (KL-divergence) to the
expected utility of the perfectly rational model. The resulting objective
function goes under various names in the literature, such as \emph{KL-control
cost} and \emph{free energy}, and has been motivated in numerous ways. The
initial inspiration comes from the maximum entropy principle of
statistical mechanics as a way to model stochastic control problems that are
linearly-solvable \cite{Fleming1977,Kappen2005a,Todorov2006}. The methods from
robust control were also adopted by economists to characterize model
misspecification \cite{Hansen2008}. Later, it was shown that the free
energy can be derived from axioms that treat utilities and information as
commensurable quantities \cite{OrtegaBraun2013}.
Then in robot reinforcement learning, researchers proposed the free
energy as a way to control the information-loss resulting from policy updates
\cite{Peters2010}. Finally, the free energy has also been motivated from
arguments that parallel rate-distortion theory and information geometry by
showing that the optimal policy minimizes the decision complexity subject to a
constraint on the expected utility \cite{Tishby2011}. It is also worth
mentioning that similar approaches arise in the context of neuroscience
\cite{Friston2009} and in game theory
\cite{Wolpert2004}.

An important yet intriguing property of the free energy is that it is
a universal aggregator of value. More precisely, it can implement the minimum,
expectation, maximization, and all the ``soft'' aggregations in between by
varying a single parameter. This has at least two important consequences. First,
the free energy corresponds, in economic jargon, to the
\emph{certainty-equivalent}, i.e.\ the value a risk-sensitive decision-maker is
thought to assign to an uncertain choice \cite{Broek2010}.
Second, it has been pointed out that decision trees
are nested certainty-equivalent operations \cite{OrtegaBraun2012}. This is
important, as decision making in the face of an adversarial, an indifferent, and
a cooperative environment, which have previously been treated as unrelated
modeling assumptions, are actually instantiations of a single decision rule.
Previous work explores this property in multi-agent settings, showing that
solutions change from being risk-dominant to payoff-dominant \cite{Kappen2012}.

\subsection{Aim of this Work}
Despite the identification of the free energy with the certainty-equivalent,
the connection to adversarial environments is as yet not fully understood. Our
work makes an important step into understanding this connection. By applying a
Legendre-Fenchel transformation to the regularization term of the free energy,
it is revealed that a single-agent free energy optimization can be thought of as
representing a game between the agent and an imaginary adversary. Furthermore,
this result explains stochastic policies as a strategy that guards the agent
from adversarial reactions of the environment.

\subsection{Structure of this Article}
This article is organized into four sections. The
aim of the next section (Section~\ref{sec:preliminaries}) is to familiarize the
reader with the mathematical foundations underlying the planning problem in AI,
and to give a basic introduction into IT bounded rationality necessary in order
to contextualize the results of our work. Section~\ref{sec:adversarial}
contains our central contribution. It first briefly reviews
Legendre-Fenchel transformations and then applies them to the free energy
functional in order to unveil the adversarial assumptions implicit in the
objective function. Finally, Section~\ref{sec:discussion} discusses the results
and concludes.

\section{Preliminaries}\label{sec:preliminaries}

We review the abstract foundations of planning under expected utility theory
and IT bounded rationality. 

\subsection{Variational Principles}

The behavior of an agent is typically characterized in one of two ways: (a) by
directly describing its policy, or (b) by specifying a \emph{variational
problem} that has the policy as its solution. While the former is a
direct specification of the agent's actions under any contingency, the latter
has the advantage that it provides an explicit (typically convex) objective
function that the policy has to optimize.
Thus, a variational principle has additional explanatory power: not only does it
single out an optimal policy, but it also encodes a \emph{preference relation}
over the set of feasible policies. Crucially, the qualifier ``optimal'' only
holds relative to the objective function and is by no means absolute: because
given \emph{any policy}, one can always engineer a variational principle that is
extremized by it. In AI, virtually all planning algorithms are framed (either
explicitly or implicitly) as \emph{maximum expected utility} problems.
This encompasses popular problem classes such as multi-armed bandits, Markov
decision processes and partially observable Markov decision processes
\cite[Chapter~3]{Legg2008}.

\subsection{Sequential versus Single-Step Decisions}

We briefly recall a basic theoretical result that will simplify our
analysis. In planning, \emph{sequential
decision problems} can be rephrased as \emph{single-step decision
problems}: instead of letting the agent choose an action in each turn, one can
equivalently let the agent choose a single \emph{policy} in the beginning that
it then must follow during its interactions with
the environment\footnote{Note that this reduction encompasses policies that
appear to change during its execution as a function of the history: the
meta-policy controlling the changes is a compressed, yet sufficient description
of the changing policy.}. This
observation was first made in game theory, where it was proven that every
\emph{extensive form game} can always be re-expressed as a \emph{normal form
game} \cite{Neumann1944}. Consequently, we abstract away from the sequential
nature of the general problem by limiting our discussion to single-step
decisions involving a single action and observation. The price we pay is to hide
the dynamical structure and to increase the complexity of the policy, but the
mathematical results can stated more concisely.

\subsection{(Subjective) Expected Utility}\label{sec:expected_utility}

Expected utility \cite{Neumann1944,Savage1954} is the \emph{de facto}
standard variational principle in artificial intelligence
\cite{RussellNorvig2010}. It is based on a \emph{utility function},
that is, a real-valued mapping of the outcomes  encoding the desirability of
each possible realization of the interactions between the agent and the
environment. The qualifiers ``subjective'' and ``expected'' in its name derive
from the fact that the desirability of a stochastic realization is calculated
as the expectation over the utilities of the particular realizations measured
with respect to the subjective beliefs of the agent.

Formally, let $\set{X}$ and $\set{Y}$ be two finite sets, the former
corresponding to the \emph{set of actions} and the latter to the \emph{set of
observations}. A \emph{realization} is a pair $(x,y) \in \set{X} \times
\set{Y}$. Furthermore, let $U: \set{X} \times \set{Y} \rightarrow \reals$ be a
utility function, such that $U(x,y)$ represents the desirability of the
realization $(x,y) \in \set{X}\times\set{Y}$; and let $q(\cdot|\cdot)$ be a
conditional probability distribution characterizing the environmental response
where $q(y|x)$ represents the probability of the observation $y \in \set{Y}$
given the action $x \in \set{X}$. Then, the agent's optimal policy $p \in
\Delta(\set{X})$ is chosen so as to maximize the functional
\begin{align}
  \expect[U]
  &= \sum_x p(x) \expect[U|x]
  \nonumber \\
  &= \sum_x p(x) \sum_y q(y|x) U(x,y).
  \label{eq:expected_utility}
\end{align}
This is illustrated in Figure~\ref{fig:eu-simplex}. An \emph{optimal policy} is
any distribution $p^\ast$ with no support over suboptimal actions,
that is, $p^\ast(x) = 0$ whenever $\expect[U|x] \leq \max_z \expect[U|z]$. In
particular, because the expected utility is linear in the policy
probabilities, one can always choose a solution that is a vertex of the
probability simplex $\Delta(\set{X})$:
\begin{equation}\label{eq:optimal_expected_utility}
  p^\ast(x)
  = \delta_{x^\ast}^x
  = \begin{cases}
    1 & \text{if $x = x^\ast := \arg\max_x \expect[U|x]$}\\
    0 & \text{otherwise,}
    \end{cases}
\end{equation}
where $\delta$ is the Kronecker delta. Hence, there always exists a
deterministic optimal policy.

\begin{figure}
\centering
\includegraphics[width=\columnwidth]{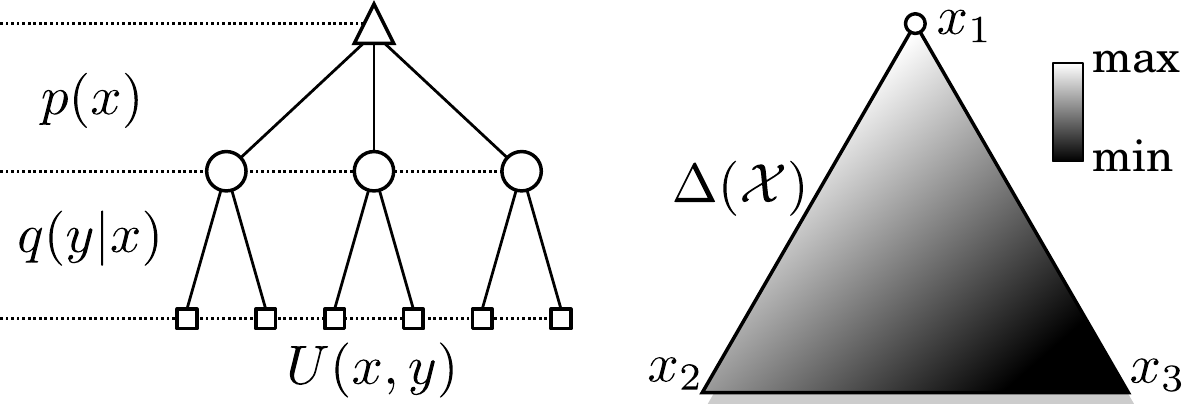}
\caption{Expected Utility Theory. The decision problem (left) boils down to
choosing a policy---a member of the simplex over actions $x \in
\mathcal{X}$---maximizing the convex combination over conditional expected
utilities $\expect[U|x]$ (right).}\label{fig:eu-simplex}
\end{figure}

Once the optimal policy has been chosen, the utility $U$ becomes a well-defined
random variable with probability distribution
\begin{align}
  \prob(U=u)
  &= \prob\bigl\{ (x,y) : U(x,y) = u \bigr\} \nonumber \\
  &= \sum_{U^{-1}(u)} p^\ast(x) q(y|x). \label{eq:utility_distribution}
\end{align}
Even though the optimal policy might be deterministic, the utility
of the ensuing realization is in general stochastic.

Let us now consider the case when the environment is another agent. Formally,
this means that the agent lacks the conditional probability distribution
$q(\cdot|\cdot)$ over observations that it requires in order the evaluate the
expected utilities. Instead, the agent possesses a second utility function $V:
\set{X} \times \set{Y} \rightarrow \reals$ characterizing the desires of the
environment. Game theory then invokes a \emph{solution concept}, most notably
the \emph{Nash equilibrium}, in order to obtain the missing distribution
$q(\cdot|\cdot)$ from $U$ and $V$ that renders the original description as a
well-defined decision problem. Thus, conceptually, game theory starts from
slightly weaker assumptions than expected utility theory. For simplicity, we
restrict ourselves to two particular cases: (a) the fully adversarial case
$U(x,y) = -V(x,y)$; and (b) the fully cooperative case $U(x,y) = V(x,y)$. The
Nash equilibrium then yields the decision rules \cite{Osborne1999}:
\begin{align}
  p^\ast &= \arg\max_p \sum_x p(x) \biggl\{ \min_q q(y|x) U(x,y) \biggr\},
  \label{eq:minmax} \\
  p^\ast &= \arg\max_p \sum_x p(x) \biggl\{ \max_q q(y|x) U(x,y) \biggr\}
  \label{eq:maxmax}
\end{align}
for the two cases respectively. Comparing these to~\eqref{eq:expected_utility}
we immediately see that, for planning purposes, it is irrelevant whether the
decision is made by the agent or not---what matters is the degree to which an
outcome (be it an action or an observation) contributes to the agent's
objective function (as encoded by one of the three possible aggregation
operators). Thus, equations~\eqref{eq:expected_utility}, \eqref{eq:minmax} and
\eqref{eq:maxmax} can be regarded as consisting of not one, but \emph{two}
nested ``decision steps'' of the form
\[
  \max_x \{ U(x) \}, \quad
  \exp_p [ U(X) ] \quad \text{or} \quad
  \min_x \{ U(x) \}.
\]
We end this brief review by summarizing the main properties of expected utility:
\begin{enumerate}
  \item \emph{Linearity:} The objective function is linear in the policy.
  \item \emph{Existence of Deterministic Solution:} Although there are optimal
  stochastic policies, there always exists an equivalent deterministic
  optimal policy.
  \item \emph{Indifference to Higher-Order Moments:} Expected utility
  places constraints on the first, but not on the higher-moments of the
  probability distribution of the utility.
  \item \emph{Exhaustive Search:} Finding the optimal deterministic policy
  requires, in the worst-case, an exhaustive evaluation of all the expected
  utilities.
\end{enumerate}
It is important to note that these properties only apply to
expected utility, but not to game theory.

\subsection{Free Energy}

How does an agent make decisions when it cannot exhaustively evaluate all
the alternatives? IT bounded rationality addresses this question by defining
a new objective function that trades off utilities versus information
costs. Planning is conceptualized as a process that transforms a prior policy
into a posterior policy that incurs a cost measured in \emph{utiles}.
Importantly, it is assumed that the agent is not allowed itself
to
reason about the costs of this transformation\footnote{The inability to reason
about resource costs renders IT bounded rationality fundamentally
different from the meta-reasoning approach of Stuart J. Russell
\cite{Russell1995a}.}; rather, it tries to optimize the original expected
utility but then runs out of resources. From the point of view of an external
observer, this ``interrupted'' planning process will appear as if it were
explicitly optimizing the free energy.

Formally, let $\set{X}$ be a finite set corresponding to the \emph{set of
realizations} and $U: \set{X} \rightarrow \reals$ is the utility function.
Furthermore, let $p_0 \in \Delta(\set{X})$ be a prior policy and $\beta \in
\reals$ be a boundedness parameter. Then, the agent's optimal policy $p \in
\Delta(\set{X})$ is chosen so as to extremize the free energy functional
\begin{equation}\label{eq:free-energy}
  F_\beta[p] :=
  \underbrace{ \sum_x p(x) U(x) }_\text{Expected Utility}
  - \underbrace{ \frac{1}{\beta} \sum_x p(x)
    \log \frac{ p(x) }{ p_0(x) } }_\text{Information Costs},
\end{equation}
where it is seen that $\beta$ controls the conversion from units of information
to utiles (Figure~\ref{fig:free-energy-simplex}). The optimal policy is given by
the ``softmax''-like distribution
\begin{equation}\label{eq:equilibrium}
  p^\ast(x)
  = \frac{ p_0(x) e^{\beta U(x)} }
    { \sum_x p_0(x) e^{\beta U(x)} }
\end{equation}
which is known as the \emph{equilibrium distribution}. Notice that $\beta$
controls how much the agent is in control of the choice: a value $\beta \approx
0$ falls back to the prior policy and represents lack of influence; and values
far away from zero yield either adversarial or cooperative choices depending on
the sign of~$\beta$. An important property of the equilibrium distribution is
that it can be simulated exactly \emph{without} evaluating all the utilities
using Monte Carlo methods \cite{Ortega2014a}.

\begin{figure}
\centering
\includegraphics[width=\columnwidth]{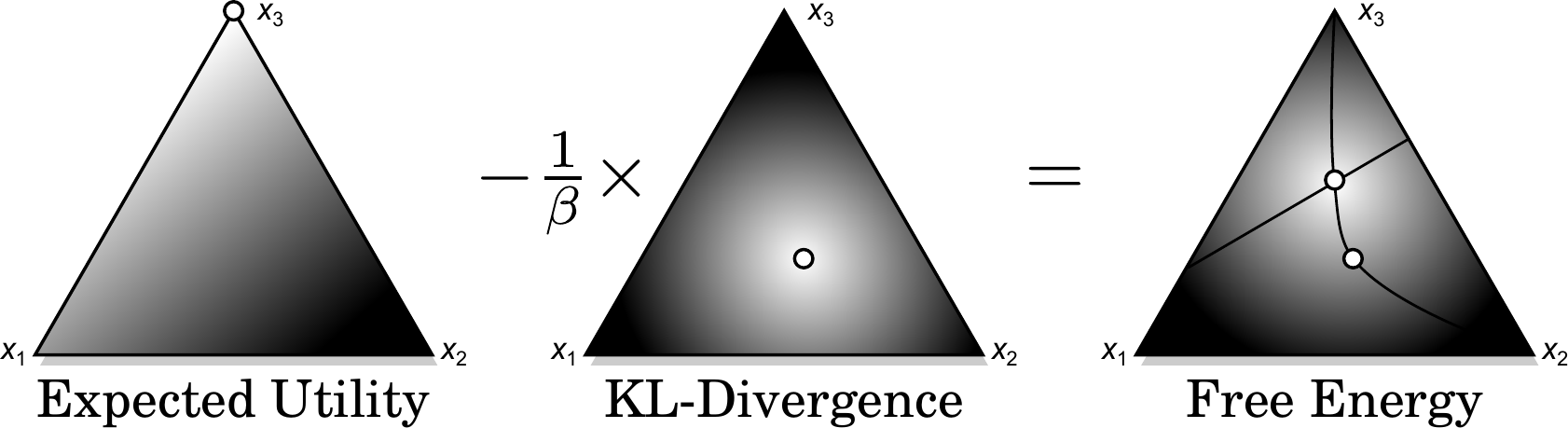}
\caption{Free energy consists of the expected utility regularized by the
KL-divergence between a given prior and posterior policy. The resulting
objective function is convex and has an optimum $p^\ast$ that is in general in
the interior of the probability simplex
$\Delta(\set{X})$.}\label{fig:free-energy-simplex}
\end{figure}

As we have mentioned before, the free energy~\eqref{eq:free-energy}
corresponds to the certainty-equivalent. This is seen as follows: the extremum
of~\eqref{eq:free-energy} given by
\begin{equation}\label{eq:extremum}
  \frac{1}{\beta} \log Z_\beta,
  \qquad \text{where} \quad
  Z_\beta := \sum_x p_0(x) e^{\beta U(x)},
\end{equation}
but now seen \emph{as a function of $\beta$}, can be thought of as an
interpolation of the maximum, expectation and minimum operator, since
\begin{align*}
  \tfrac{1}{\beta} \log Z_\beta &= \max_x \{ U(x) \}
    & \beta &\rightarrow +\infty \\
  \tfrac{1}{\beta} \log Z_\beta &= \quad \mathbf{E}_{p_0}[ U ]
    & \beta &\rightarrow \phantom{+}0 \\
  \tfrac{1}{\beta} \log Z_\beta &= \min_x \{ U(x) \}
    & \beta &\rightarrow -\infty.
\end{align*}
We end this brief review by summarizing the main properties of IT bounded
rationality:
\begin{enumerate}
  \item \emph{Nonlinearity:} The objective function is nonlinear in the policy.
  \item \emph{Stochastic Solutions:} Optimal policies are inherently
  stochastic.
  \item \emph{Sensitive to Higher-Order Moments:} Because the free energy is
  nonlinear in the policy, it places constraints on the higher-order moments
  of the policy too. In particular, a Taylor expansion of the KL-divergence term
  reveals that the free energy is sensitive to \emph{all} the moments of the
  policy.
  \item \emph{Randomized Search:} Obtaining a decision from the optimal policy
  does not require the exhaustive evaluation of the utilities. Rather, it can
  be sampled using Monte Carlo techniques.
\end{enumerate}

\begin{figure}
\centering
\includegraphics[width=0.9\columnwidth]{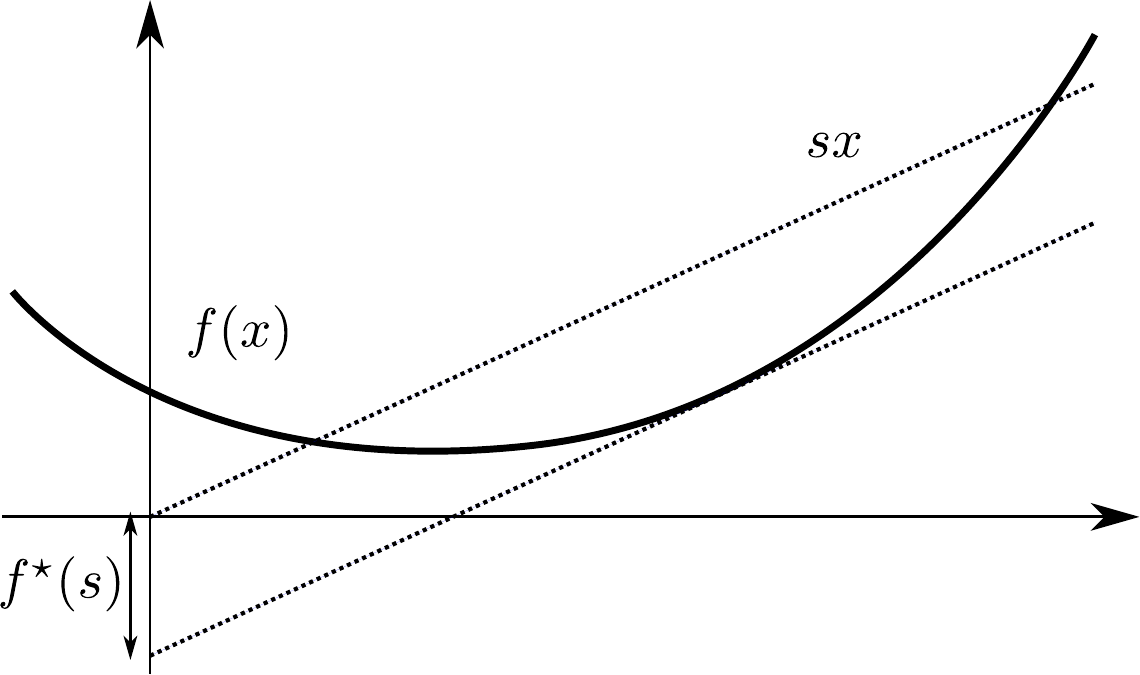}
\caption{The Legendre-Fenchel Transform. Given a function
$f(x)$, the convex conjugate $f^\star(s)$ corresponds to the intercept of a
tangent line to the curve with slope $s$.}\label{fig:legendre}
\end{figure}

\section{Adversarial Interpretation}\label{sec:adversarial}

Before presenting our main results, we give a brief definition of
Legendre-Fenchel transforms.

\subsection{Legendre-Fenchel Transforms}

The \emph{Legendre-Fenchel transformation} is a
transformation that yields an alternative encoding of a given functional
relationship. More precisely, the information contained in the function is
expressed using the derivative as the independent variable. Because of this,
Legendre-Fenchel transformations play an important role in thermodynamics and
optimization. For a function $f(x) : \reals^n \rightarrow \reals$, it is defined
by the variational formula
\[
  f^\star(s) = \sup_{x \in \set{X}} \{ \langle s, x \rangle - f(x) \}
\]
where $f^\star(s): \reals^n \rightarrow \reals$ is the \emph{convex
conjugate} and $\langle s, x\rangle$ is the inner product of two vectors $s,x$
in $\reals^n$. The following are common examples:
\begin{enumerate}
\item \emph{Affine function:}
\[
   f(x) = ax - b
   \qquad
   f^\star(s) =
   \begin{cases}
   b & \text{if $s = a$,} \\
   +\infty & \text{if $s \neq a$.}
   \end{cases}
\]
\item \emph{Power function:}
\[
  f(x) = \frac{1}{\alpha} |x|^\alpha
  \qquad
  f^\star(s) = \frac{1}{\alpha'} |s|^{\alpha'}
\]
where $\frac{1}{\alpha} + \frac{1}{\alpha'} = 1$.
\item \emph{Exponential function:}
\[
  f(x) = e^x
  \qquad
  f^\star(s) =
  \begin{cases}
  s \log s - s & \text{if $s > 0$,} \\
  0 & \text{if $s = 0$,} \\
  +\infty & \text{if $s < 0$.}
  \end{cases}
\]
\end{enumerate}
We refer the reader to Touchette's article \cite{Touchette2005} for a brief
introduction or Boyd and Vanderberghe's book \cite{Boyd2004} for a thorough
treatment.

\subsection{Unveiling the Adversarial Environment}
Using a Legendre-Fenchel transformation, one can show that the regularization
term of the free energy can be rewritten as a variational problem.
\begin{lemma}\label{lemma:kl}
\begin{multline}
  -\frac{1}{\beta} \sum_x p(x) \log \frac{ p(x) }{ p_0(x) }
  \\= \min_{C} \sum_x -p(x) C(x) + p_0(x) e^{\beta C(x)}-\frac{1}{\beta}(\log
\beta+1)
\end{multline}
\end{lemma}
\begin{proof}
Since $n = |\set{X}|$ is finite, $p(x), C(x)$ are vectors in~$\reals^n$.
Essentially, the lemma says that the l.h.s. is the convex conjugate of
the function
\[
    f(C) = -\sum_x p_0(x) e^{\beta C(x)} + \frac{1}{\beta}(\log \beta + 1),
\]
because the r.h.s. can be re-expressed as
\[
  \min_C \sum_x -p(x)C(x) + f(C) = \max_C \sum_x p(x)C(x) - f(C).
\]
Setting the derivative w.r.t. $C$ to zero yields the optimality condition
\[
  p(x) = \beta p_0(x) e^{\beta C(x)}.
\]
Isolating $C(x)$ and substituting into $\sum_x p(x)C(x) - f(C)$ proves the
lemma.
\end{proof}

Consequently, the regularization term of the free energy encapsulates a
second variational problem over an auxiliary vector $C$. In particular,
the logarithmic form encodes an objective function that is exponential in
$C$. Equipped with this lemma, we can state our main result.

\begin{theorem}
The maximization of the free energy~\eqref{eq:free-energy} is
equivalent to the maximization of
\begin{equation}\label{eq:adversarial}
  \min_C
  \underbrace{ \sum_x p(x) [U(x) - C(x)] }_\text{Expected Net Utility}
  + \underbrace{ \sum_x p_0(x) e^{\beta C(x)} }_\text{Penalty of Adversary}
\end{equation}
w.r.t. the policy $p \in \Delta(\set{X})$.
\end{theorem}
\begin{proof}
The proof is an immediate consequence of Lemma~\ref{lemma:kl}:
\begin{small}
\begin{align*}
  &\arg\max_p \biggl\{ \sum_x p(x)U(x)
  - \frac{1}{\beta} \sum_x p(x) \log \frac{p(x)}{p_0(x)} \biggr\}
  \\
  &= \arg\max_p \biggl\{ \sum_x p(x)U(x)
  + \min_{C} \biggl\{ \sum_x -p(x) C(x)
  \\&\qquad + p_0(x) e^{\beta C(x)}
    - \frac{1}{\beta}(\log \beta+1)\biggr\} \biggr\}
  \\
  &= \arg\max_p\min_C\biggl\{ \sum_x p(x)[U(x) - C(x)]
  + \sum_x p_0(x) e^{\beta C(x)}\biggr\}.
\end{align*}
\end{small}
\end{proof}

This result can be interpreted as follows. When a single agent maximizes the
free energy, it is implicitly assuming a situation where it is playing against
an imaginary adversary. In this situation, the agent first chooses
its policy $p \in \Delta(\set{X})$, and then the adversary
attempts to decrease the agent's net utilities by subtracting
costs $C(x)$. However, the adversary cannot choose these costs arbitrarily;
instead, it must pay exponential penalties. In fact, its optimal strategy can be
directly calculated from Lemma~\ref{lemma:kl}.
\begin{corollary}
The imaginary adversary's optimal strategy is given by
\begin{equation}\label{eq:optimal-cost}
  C^\ast(x) = \frac{1}{\beta} \log \frac{p(x)}{\beta p_0(x)}.
\end{equation}
\end{corollary}
Inspecting \eqref{eq:optimal-cost}, we see that the optimal costs scale
relatively with the agent's deviation from its prior probabilities. This
leads to an interesting interaction between the agent's and the imaginary
adversary's choices.

\subsection{Indifference}

Our second main result characterizes the solution to this adversarial setup. It
turns out that the adversary's best strategy is to choose costs such that the
agent's net payoffs are \emph{uniform}.

\begin{theorem}\label{theo:indifference}
The solution to
\[
  \max_p \min_C \sum_x p(x) [U(x) - C(x)] + \sum_x p_0(x) e^{\beta C(x)}
\]
has the property that for all $x \in \set{X}$,
\[
  U(x) - C(x) = \text{constant}.
\]
\end{theorem}
\begin{proof}
We first note that we can exchange the maximum and minimum operations,
\begin{align*}
&\max_p \min_C \sum_x p(x) [U(x) - C(x)] + \sum_x p_0(x) e^{\beta C(x)}
\\&=\min_C \max_p \sum_x p(x) [U(x) - C(x)] + \sum_x p_0(x) e^{\beta C(x)}
\end{align*}
because there is no duality gap due to the concavity of the exponential
function. We can thus maximize first w.r.t. the policy $p$. Define $\set{X}^\ast
\subset \set{X}$ as the subset of elements maximizing the penalized
expected utility, that is, for all $x^\ast \in \set{X}$ and $x \in \set{X}$,
\begin{equation}\label{eq:maximal-set}
  U(x^\ast)-C(x^\ast) \geq U(x)-C(x).
\end{equation}
Maximizing w.r.t. to $p$ yields optimal probabilities $p^\ast(x)$ given by
\[
  p^\ast(x) = \begin{cases}
    q(x) & \text{if $x \in \set{X}^\ast$,}\\
    0 & \text{otherwise,}
  \end{cases}
\]
where $q$ is any distribution over $\set{X}^\ast$. Given this, the worst case
costs $C^\ast(x)$ are
\begin{equation}\label{eq:opt-costs-2}
  C^\ast(x) = \begin{cases}
     \frac{1}{\beta}\log\frac{q(x)}{\beta p_0(x)}
      & \text{if $x\in\set{X}^\ast$,}\\
     -\infty & \text{otherwise.}
  \end{cases}
\end{equation}
However, if $\set{X}^\ast \neq \set{X}$, then we get a contradiction, since
\[
  U(x^\ast) - C^\ast(x^\ast) \not\geq U(x) - C^\ast(x)
\]
for all $x \notin \set{X}^\ast$, violating~\eqref{eq:maximal-set}. Hence, it
must be that for all $x \in \set{X}$,
\[
  U(x)-C(x) = \text{constant,}
\]
concluding the proof. 
\end{proof}

To get a better understanding on the meaning of this result, it is helpful to consider an
example. Figure~\ref{fig:minimax} illustrates four choices of policies and the
corresponding optimal adversarial costs. Here it is seen that an agent
can protect itself by spreading the probability mass of its policy over many
realizations.

\begin{figure}
\centering
\includegraphics[width=\columnwidth]{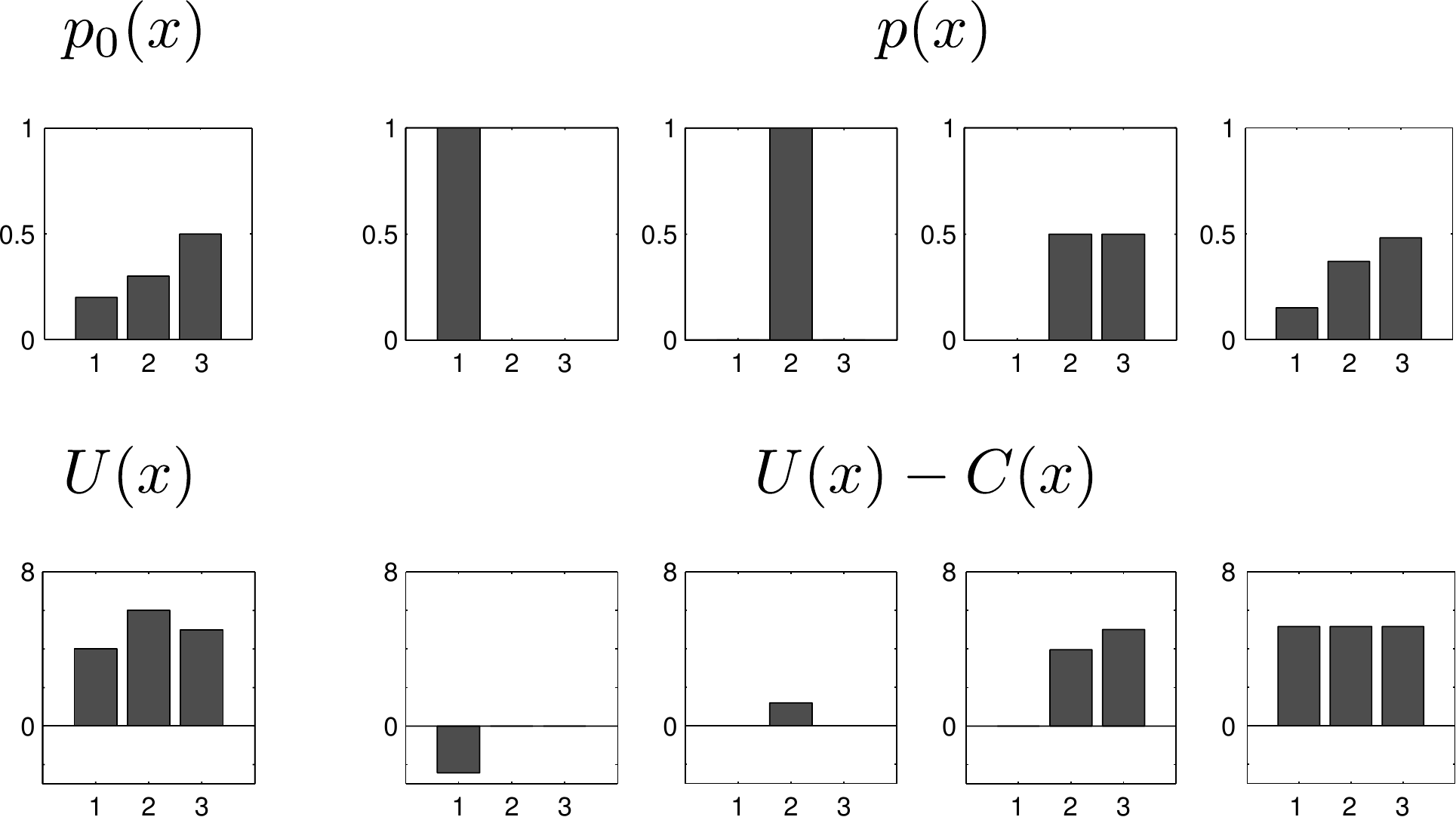}
\caption{Given an agent's policy $p$, the environment chooses costs that change
the net utilities $U(x)-C(x)$ of the realizations with support. The agent
can protect itself from adversarial costs by randomization.
The last column corresponds to the agent's optimal policy.}\label{fig:minimax}
\end{figure}

\section{Discussion}\label{sec:discussion}

Starting from the free energy functional, we have shown
how to construct an alternative adversarial interpretation that constitutes an
equivalent problem. Conceptually, our findings can be summarized as follows:
\begin{enumerate}
\item A regularization of the expected utility encodes assumpions about
deviations from the expected utility. 
\item The Legendre-Fenchel transformation reinterprets the free energy
as a game against an environmental adversary.
\item In this transformation, it is seen that the regularization is equivalent
to the penalization of adversarial costs.
\item Stochastic policies guard against adversarial costs. Expected utility
alone is linear in the policy and thus encodes deterministic optimal policies.
\end{enumerate}

\subsection{Indifference and Nash Equilibrium}
Our results establish an interesting relation to
game theory. Theorem~\ref{theo:indifference}
and~\eqref{eq:optimal-cost} immediately imply
\[
  U(x)-C^\ast(x) =  U(x) - \frac{1}{\beta} \log \frac{p(x)}{\beta p_0(x)} =
\text{constant}
\]
for all $x \in \set{X}$. This turns out to be an known characterization of
the equilibrium distribution \eqref{eq:equilibrium}. However, in light of
our present results, it acquires a different twist; it fits with the
well-known result that a Nash equilibrium is a strategy profile such that each
player chooses a (mixed) strategy that renders the others players indifferent to
their choices \cite{Osborne1999}.

\subsection{Other Regularizers}

The presented method is general and can also be applied to other regularizers
in order to make the assumptions about the imaginary adversary explicit.
For instance, the following list enumerates some examples.
\begin{enumerate}
\item \emph{Expected Utility:} Because the regularization term is null, the
resulting adversary's objective function is
\[
  \max_p \min_C \sum_x p(x) [U(x)-C(x)] - \delta^{C(x)}_0.
\]
Here we see that a null regularization implies an adversary devoid of power,
i.e.\ one that cannot alter the utilities chosen by the agent.
\item \emph{Power Function:} If the regularizer is a power function, then we
have the dual power function
\[
  \max_p \min_C \sum_x p(x) [U(x)-C(x)] + |C(x)|^\alpha.
\]
This case is interesting because it encompases ridge and lasso as special cases.
\item \emph{Modern Portfolio Theory:} One notable case that is widely used in
practice is Modern Portfolio Theory. Here, an investor trades off asset
returns versus portfolio risk, encoded into a regularization term that is
quadratic in the policy \cite{Markowitz1952}. The transformed objective
function is given by
\[
  \max_p \min_C \sum_x p(x) [U(x)-C(x)] + \frac{1}{2}\lambda C^T \Sigma C.
\]
\end{enumerate}

\subsection{Conclusions}

Our central result shows how to extract the adversarial cost function
implicitly assumed in the agent's objective function by means
of an appropriately chosen Legendre transformation. Conversely, one
can also start from the cost function of an adversarial environment
and reexpress it as a regularized optimization. While our main
motivation was to apply this to the free energy functional, the
transformation is general and works also in other well-known
frameworks such as in modern portfolio theory.
The immediate application is that we can switch between the two representations
and pick the one that is easier to solve. This suggests novel algorithms for
solving the single-agent planning problem based on ideas from differential game
theory \cite{Dockner2001} and convex programming~\cite{Zinkevich2003}.
Furthermore, this also suggests that agents that randomize their decisions, are
essentially expressing their information limitations by protecting themselves
from undesired outcomes.  As such, we believe this is a basic result that opens
up a series of additional questions regarding the nature of stochastic policies,
such as the conditions under which they arise and how they encode risk
sensitivity, which need to be further explored in the future.


\newpage

\subsection*{Acknowledgements}
We thank the anonymous reviewers for their valuable comments and suggestions for
improving this manuscript. This study was funded by grants from the
U.S. National Science Foundation, Office of Naval Research and
Department of Transportation.



\begin{thebibliography}{}

\bibitem[\protect\citeauthoryear{Boyd and Vandenberghe}{2004}]{Boyd2004}
Boyd, S., and Vandenberghe, L.
\newblock 2004.
\newblock {\em {Convex Optimization}}.
\newblock Cambridge Univeristy Press.

\bibitem[\protect\citeauthoryear{Browne \bgroup et al\mbox.\egroup
  }{2012}]{Browne2012}
Browne, C.; Powley, E.; Whitehouse, D.; Lucas, S.; Cowling, P.; Rohlfshagen,
  P.; Travener, S.; Perez, D.; Samothrakis, S.; and Colton, S.
\newblock 2012.
\newblock {A Survey of Monte Carlo Tree Search Methods}.
\newblock {\em IEEE Transactions on Computational Intelligence and AI in Games}
  4(1).

\bibitem[\protect\citeauthoryear{Coulom}{2006}]{Coulom2006}
Coulom, R.
\newblock 2006.
\newblock {Efficient Selectivity and Backup Operators in Monte-Carlo Tree
  Search}.
\newblock In {\em {Computer and Games}}.

\bibitem[\protect\citeauthoryear{Dockner \bgroup et al\mbox.\egroup
  }{2001}]{Dockner2001}
Dockner, E.; Jorgensen, S.; Long, N.; and Sorger, G.
\newblock 2001.
\newblock {\em {Differential Games in Economics and Management Science}}.
\newblock Cambridge University Press.

\bibitem[\protect\citeauthoryear{Fleming}{1977}]{Fleming1977}
Fleming, W.
\newblock 1977.
\newblock {Exit Probabilities and Optimal Stochastic Control}.
\newblock {\em Applied Mathematics and Optimization} 4:329--346.

\bibitem[\protect\citeauthoryear{Friston}{2009}]{Friston2009}
Friston, K.
\newblock 2009.
\newblock {The free-energy principle: a rough guide to the brain?}
\newblock {\em Trends in Cognitive Science} 13:293--301.

\bibitem[\protect\citeauthoryear{Gelly \bgroup et al\mbox.\egroup
  }{2006}]{Gelly2006}
Gelly, S.; Wang, Y.; Munos, R.; and Teytaud, O.
\newblock 2006.
\newblock {Modification of UCT with Patterns in Monte-Carlo Go}.
\newblock Technical report, Inst. Nat. Rech. Inform. Auto. (INRIA).

\bibitem[\protect\citeauthoryear{Hansen and Sargent}{2008}]{Hansen2008}
Hansen, L., and Sargent, T.
\newblock 2008.
\newblock {\em {Robustness}}.
\newblock Princeton: Princeton University Press.

\bibitem[\protect\citeauthoryear{Hutter}{2004}]{Hutter2004}
Hutter, M.
\newblock 2004.
\newblock {\em {Universal Artificial Intelligence: Sequential Decisions based
  on Algorithmic Probability}}.
\newblock Berlin: Springer.

\bibitem[\protect\citeauthoryear{Kappen, G{\'o}mez, and
  Opper}{2012}]{Kappen2012}
Kappen, H.; G{\'o}mez, V.; and Opper, M.
\newblock 2012.
\newblock {Optimal control as a graphical model inference problem}.
\newblock {\em Machine Learning} 1:1--11.

\bibitem[\protect\citeauthoryear{Kappen}{2005a}]{Kappen2005a}
Kappen, H.
\newblock 2005a.
\newblock {A linear theory for control of non-linear stochastic systems}.
\newblock {\em Physical Review Letters} 95:200201.

\bibitem[\protect\citeauthoryear{Kappen}{2005b}]{Kappen2005b}
Kappen, H.
\newblock 2005b.
\newblock {Path integrals and symmetry breaking for optimal control theory}.
\newblock {\em Journal of Statistical Mechanics: Theory and Experiment}.

\bibitem[\protect\citeauthoryear{Legg}{2008}]{Legg2008}
Legg, S.
\newblock 2008.
\newblock {\em {Machine Super Intelligence}}.
\newblock Ph.D. Dissertation, Department of Informatics, University of Lugano.

\bibitem[\protect\citeauthoryear{Markowitz}{1952}]{Markowitz1952}
Markowitz, H.
\newblock 1952.
\newblock {Portfolio Selection}.
\newblock {\em The Journal of Finance} 7:77--�91.

\bibitem[\protect\citeauthoryear{Mnih \bgroup et al\mbox.\egroup
  }{2013}]{Mnih2013}
Mnih, V.; Kavukcuoglu, K.; Silver, D.; Graves, A.; Antonoglou, I.; Wierstra,
  D.; and Riedmiller, M.
\newblock 2013.
\newblock {Playing Atari with Deep Reinforcement Learning}.
\newblock {\em ArXiv} (1312.5602).

\bibitem[\protect\citeauthoryear{Ortega and Braun}{2012}]{OrtegaBraun2012}
Ortega, P., and Braun, D.
\newblock 2012.
\newblock {Free Energy and the Generalized Optimality Equations for Sequential
  Decision Making}.
\newblock In {\em {European Workshop on Reinforcement Learning (EWRL'10)}}.

\bibitem[\protect\citeauthoryear{Ortega and Braun}{2013}]{OrtegaBraun2013}
Ortega, P.~A., and Braun, D.~A.
\newblock 2013.
\newblock {Thermodynamics as a Theory of Decision-Making with Information
  Processing Costs}.
\newblock {\em Proceedings of the Royal Society A 20120683}.

\bibitem[\protect\citeauthoryear{Ortega, Braun, and Tishby}{2014}]{Ortega2014a}
Ortega, P.; Braun, D.; and Tishby, N.
\newblock 2014.
\newblock {Monte Carlo Methods for Exact \& Efficient Solution of the
  Generalized Optimality Equations}.
\newblock In {\em {IEEE International Conference on Robotics and Automation
  (ICRA)}}.

\bibitem[\protect\citeauthoryear{Osborne and Rubinstein}{1999}]{Osborne1999}
Osborne, M., and Rubinstein, A.
\newblock 1999.
\newblock {\em {A Course in Game Theory}}.
\newblock {MIT} Press.

\bibitem[\protect\citeauthoryear{Papadimitriou and
  Tsitsiklis}{1987}]{PapadimitriouTsitsiklis1987}
Papadimitriou, C., and Tsitsiklis, J.
\newblock 1987.
\newblock {The Complexity of Markov Decision Processes}.
\newblock {\em Mathematics of Operations Research} 12(3):441--450.

\bibitem[\protect\citeauthoryear{Peters, M{\"u}lling, and
  Alt{\"u}n}{2010}]{Peters2010}
Peters, J.; M{\"u}lling, K.; and Alt{\"u}n, Y.
\newblock 2010.
\newblock {Relative entropy policy search}.
\newblock In {\em {AAAI}}.

\bibitem[\protect\citeauthoryear{Rubinstein}{1998}]{Rubinstein1998}
Rubinstein, A.
\newblock 1998.
\newblock {\em {Modeling Bounded Rationality}}.
\newblock Cambridge, MA: {MIT} Press.

\bibitem[\protect\citeauthoryear{Russell and Norvig}{2010}]{RussellNorvig2010}
Russell, S., and Norvig, P.
\newblock 2010.
\newblock {\em {Artificial Intelligence: A Modern Approach}}.
\newblock Prentice-Hall, Englewood Cliffs, NJ, 3rd edition edition.

\bibitem[\protect\citeauthoryear{Russell}{1995}]{Russell1995a}
Russell, S.
\newblock 1995.
\newblock {Rationality and Intelligence}.
\newblock In Mellish, C., ed., {\em {Proceedings of the Fourteenth
  International Joint Conference on Artificial Intelligence}},  950--957.
\newblock San Francisco: Morgan Kaufmann.

\bibitem[\protect\citeauthoryear{Savage}{1954}]{Savage1954}
Savage, L.
\newblock 1954.
\newblock {\em {The Foundations of Statistics}}.
\newblock New York: John Wiley and Sons.

\bibitem[\protect\citeauthoryear{Simon}{1972}]{Simon1972}
Simon, H.
\newblock 1972.
\newblock {Theories of Bounded Rationality}.
\newblock In Radner, C., and Radner, R., eds., {\em {Decision and
  Organization}}. Amsterdam: North Holland Publ.
\newblock  161--176.

\bibitem[\protect\citeauthoryear{Strens}{2000}]{Strens2000}
Strens, M.
\newblock 2000.
\newblock {A Bayesian Framework for Reinforcement Learning}.
\newblock In {\em {ICML}}.

\bibitem[\protect\citeauthoryear{Theodorou, Buchli, and
  Schaal}{2010}]{Theodorou2010}
Theodorou, E.; Buchli, J.; and Schaal, S.
\newblock 2010.
\newblock {A generalized path integral approach to reinforcement learning}.
\newblock {\em Journal of Machine Learning Research} 11:3137--3181.

\bibitem[\protect\citeauthoryear{Tishby and Polani}{2011}]{Tishby2011}
Tishby, N., and Polani, D.
\newblock 2011.
\newblock {\em {Perception-Action Cycle}}.
\newblock Springer New York.
\newblock chapter Information Theory of Decisions and Actions,  601--636.

\bibitem[\protect\citeauthoryear{Todorov}{2006}]{Todorov2006}
Todorov, E.
\newblock 2006.
\newblock {Linearly solvable Markov decision problems}.
\newblock In {\em {Advances in Neural Information Processing Systems}},
  volume~19,  1369--1376.

\bibitem[\protect\citeauthoryear{Touchette}{2005}]{Touchette2005}
Touchette, H.
\newblock 2005.
\newblock {Legendre-Fenchel Transforms in a Nutshell}.
\newblock Technical report, Rockefeller University.

\bibitem[\protect\citeauthoryear{van~den Broek, Wiegerinck, and
  Kappen}{2010}]{Broek2010}
van~den Broek, B.; Wiegerinck, W.; and Kappen, H.
\newblock 2010.
\newblock {Risk Sensitive Path Integral Control}.
\newblock In {\em {UAI}},  615--622.

\bibitem[\protect\citeauthoryear{Veness \bgroup et al\mbox.\egroup
  }{2011}]{Veness2011}
Veness, J.; Ng, M.; Hutter, M.; Uther, W.; and Silver, D.
\newblock 2011.
\newblock {A Monte-Carlo AIXI Approximation}.
\newblock {\em Journal of Artificial Intelligence Research} 40:95--142.

\bibitem[\protect\citeauthoryear{{Von Neumann} and
  Morgenstern}{1944}]{Neumann1944}
{Von Neumann}, J., and Morgenstern, O.
\newblock 1944.
\newblock {\em {Theory of Games and Economic Behavior}}.
\newblock Princeton: Princeton University Press.

\bibitem[\protect\citeauthoryear{Wolpert}{2004}]{Wolpert2004}
Wolpert, D.
\newblock 2004.
\newblock {\em {Complex Engineering Systems}}.
\newblock Perseus Books.
\newblock chapter Information theory - the bridge connecting bounded rational
  game theory and statistical physics.

\bibitem[\protect\citeauthoryear{Zinkevich}{2003}]{Zinkevich2003}
Zinkevich, M.
\newblock 2003.
\newblock {Online Convex Programming and Generalized Infinitesimal Gradient
  Ascent}.
\newblock In {\em {ICML}},  928--936.

\end{thebibliography}
\bibliographystyle{aaai}

\end{document}